\RequirePackage[svgnames]{xcolor}
\documentclass[11pt,letterpaper]{mystyle}
\usepackage[comma,authoryear,compress]{natbib}
\usepackage[T1]{fontenc}
\usepackage[utf8]{inputenc}
\usepackage{amsmath,amssymb,amsthm}
\usepackage{graphicx}
\usepackage{booktabs,multirow,array}
\usepackage{xcolor,colortbl}
\usepackage{adjustbox}
\usepackage{float}
\usepackage{placeins}
\usepackage{hyperref}
\usepackage{url}
\usepackage{xurl}
\usepackage{eso-pic}
\definecolor{ProCreditInk}{HTML}{293F59}
\definecolor{ProCreditPanel}{HTML}{F3F6FA}
\colorlet{TinaCrimson}{ProCreditInk}
\renewcommand{\titlefont}{\raggedright\color{ProCreditInk}\normalfont\fontsize{19}{23}\selectfont}
\tcbset{titlebox/.append style={colback=ProCreditPanel,colframe=ProCreditPanel}}
\definecolor{tablegroup}{gray}{0.93}
\newcommand{\score}[2]{\mbox{#1\textsubscript{\normalfont\scriptsize\ensuremath{\pm}#2}}}
\newcommand{\bestscore}[2]{\score{\textbf{#1}}{#2}}
\newcommand{\secondscore}[2]{\score{\underline{#1}}{#2}}
\newcommand{\scoreonly}[1]{\mbox{#1}}

\hypersetup{
  colorlinks=true,
  allcolors=YaleBlue,
  pdftitle={ProCredit: From Outcome Rewards to Progress Credit in Agentic Reinforcement Learning},
  pdfauthor={Ming Ma, Yi Zhu, Yiran Zhong, Feida Zhu, Chonghan Liu, Pengkun Jiao, Qichao Wang, Yanhao Jia, Tianming Yang, Steven Hoi}
}
\newtheorem{proposition}{Proposition}
\title{ProCredit: From Outcome Rewards to\\Progress Credit in Agentic Reinforcement Learning}
\runningtitle{ProCredit: From Outcome Rewards to Progress Credit in Agentic Reinforcement Learning}
\renewcommand\Affilfont{\centering\normalfont\fontsize{10}{14}\selectfont}
\makeatletter
\renewcommand\AB@affilsepx{\protect\\\protect\Affilfont}
\makeatother
\author[ ]{%
  Ming Ma$^{1,2}$, Yi Zhu$^{3,*}$, Yiran Zhong$^{3,*}$, Feida Zhu$^3$, Chonghan Liu$^4$,\\
  Pengkun Jiao$^3$, Qichao Wang$^5$, Yanhao Jia$^5$, Tianming Yang$^1$, Steven Hoi$^3$
}
\affil[1]{Institute of Neuroscience, Chinese Academy of Sciences}
\affil[2]{University of Chinese Academy of Sciences}
\affil[3]{Tongyi Lab, Alibaba Group}
\affil[4]{University of California, Los Angeles}
\affil[5]{Nanyang Technological University}
\affil[ ]{\texttt{mam2022@ion.ac.cn}, \texttt{zhu.yee@outlook.com}, \texttt{zhongyiran@gmail.com}}
\correspondingauthor{Yi Zhu (\href{mailto:zhu.yee@outlook.com}{zhu.yee@outlook.com}); Yiran Zhong (\href{mailto:zhongyiran@gmail.com}{zhongyiran@gmail.com}). $^{*}$ Corresponding authors.}
\date{}
\begin{document}
\AddToShipoutPictureFG*{%
  \AtPageUpperLeft{%
    \put(\LenToUnit{1.9cm},\LenToUnit{-1.55cm}){\includegraphics[height=0.9cm]{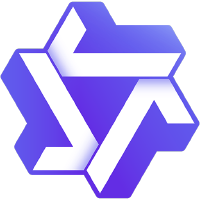}}%
    \put(\LenToUnit{1.9cm},\LenToUnit{-1.8cm}){\rule{\textwidth}{0.5pt}}%
  }%
}
\begin{abstract}
Long-horizon agentic tasks require an agent to modify an environment through a sequence of tool calls, with success determined by the final state. The standard recipe assigns a single outcome reward at the end and compares trajectories sampled for the same task. As a result, a group with no successful trajectory yields no training signal, failed attempts cannot be told apart by how close they came to completion, and turns that advance the task receive the same credit as turns that only query the environment. Prior work refines the unit of comparison from the trajectory to the step, or trains a reward model to supply intermediate signal: the former still derives its signal from final success alone, and the latter estimates it with a model. We observe that the acceptance checks that decide success can also be run on intermediate states, so progress is as verifiable as the outcome. We propose ProCredit, which turns this verified progress into credit: it reruns the acceptance checks after each turn, rewards the turn by its change in progress, and uses these rewards to assign credit both across attempts at the same task and across the turns within a trajectory. Starting from Qwen3.5 base models at three scales on AppWorld, ProCredit outperforms outcome-reward baselines and progress-based baselines in task completion rate at every scale on both test sets, exceeding the strongest outcome-reward baseline by 4.1 percentage points at 4B, and results in a second environment show the same direction of improvement. Ablations show that adding the final progress to the trajectory score alone does not improve performance: the gain comes from crediting progress to the turn where it occurs.
\par\smallskip
\raisebox{-0.2em}{\includegraphics[height=1.05em]{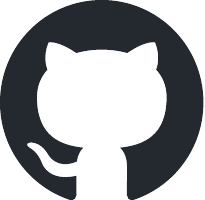}}\enspace
\makebox[6.5em][l]{\textbf{GitHub}}\url{https://github.com/MBJinX/ProCredit}\par
\raisebox{-0.2em}{\includegraphics[height=1.05em]{figures/tongyi-logo.png}}\enspace
\makebox[6.5em][l]{\textbf{Project Page}}\url{https://mbjinx.github.io/procredit_web/}
\end{abstract}

\maketitle
\vspace{3mm}

\section{Introduction}
\label{sec:introduction}

Language model agents tackle long-horizon tasks through dozens of operations across tools, applications, and services until the environment satisfies a user's request \citep{appworld,taubench}. We study tasks in which an agent gradually modifies the environment through tool calls and success depends on the final state. Consider a request to move a wake-up alarm earlier and disable the other alarms: acceptance consists of executable checks, each testing a concrete fact about the environment, such as whether the wake-up alarm has the requested time.

A common way to train such agents is reinforcement learning with verifiable rewards \citep{deepseekr1}: sample multiple trajectories for the same task, assign each a terminal outcome reward, and compare rewards within the group \citep{grpo,rloo,dapo}. It has been applied to search, tool use, and tasks spanning applications \citep{searchr1,toolrl,loop}. An outcome reward reduces an entire attempt to one number and discards three kinds of information. First, in groups with no successful trajectory, every advantage is zero, so the task contributes no gradient at that update. Second, in mixed groups, a failed trajectory that completes half the task receives the same negative advantage as one that completes none of it. Third, a turn that advances the task shares its advantage with a turn that only queries the environment. The first two losses concern how much progress is made; the third concerns when it is made (Figure~\ref{fig:problem}).

\begin{figure}[t]
    \centering
    \includegraphics[width=0.95\linewidth]{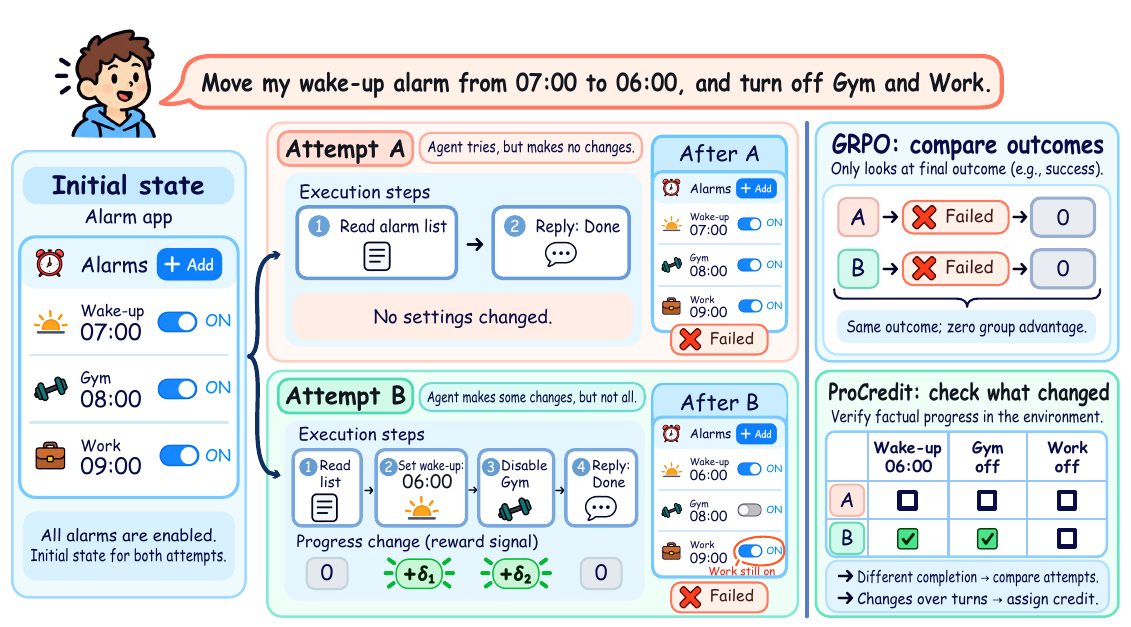}
    \caption{The same failed outcome hides different progress. Attempt A only reads the alarm list and declares completion; Attempt B moves the wake-up alarm earlier and disables Gym but leaves Work enabled. Both receive zero outcome reward and zero GRPO advantage; rerunning the checks after each turn separates them, with rewards $\delta_1$ and $\delta_2$ for changing the time and disabling Gym.}
    \label{fig:problem}
\end{figure}

Prior work recovers intermediate signals in two ways. One refines the unit of comparison from the trajectory to the step \citep{gigpo,salt}, but the signal still comes from final success alone, and two trajectories can only be compared at a step where they reach the same state. The other trains a separate model to supply the signal \citep{agentprm,toolverse}, which is estimated rather than verified in the environment. The acceptance checks that decide success can also run on intermediate states: progress is as verifiable as the outcome, and no separate model is needed.

We propose ProCredit; Figure~\ref{fig:method} shows its training procedure. After each turn during training, it reruns the task's acceptance checks, records the fraction passed as progress, and uses the change between consecutive turns as the turn's reward. Credit is assigned at two levels. Across attempts at the same task, trajectories are compared by their total score, so failed attempts are separated by completion; within a trajectory, turns are compared by their turn-level returns, so a turn that advances the task earns more credit than every turn after it. All-fail groups therefore still carry signal, and credit within a trajectory follows its own progress without requiring another trajectory to reach the same state. The procedure uses only the task's existing checks; progress is read only during training.

We train Qwen3.5 base models at three scales on AppWorld and evaluate on its two test sets, one with applications seen in training and one containing applications unseen in training. ProCredit achieves the highest task completion rate at every scale on both test sets; the direction holds on ToolSandbox, and function calling on BFCL does not decline. Two ablations locate the gain. Adding final progress to the trajectory score or spreading it over tokens does not improve at any scale, so the gain comes from crediting progress at the turn where it occurs, not only from how much is completed at the end. Discounting the turn-level return shortens trajectories and loses hard tasks along with them; undiscounted returns are better at all three scales.

\section{Related work}
\label{sec:related-work}

\paragraph{Reinforcement learning for long-horizon agentic tasks.}
Reinforcement learning allows language model agents to learn directly from environment interactions in search and tool use \citep{searchr1,toolrl}, text games and embodied tasks \citep{ragen,gigpo}, and tasks spanning applications in environments such as AppWorld \citep{loop,g2po,salt}. These methods train with relative advantages across trajectories sampled for the same task, with their main reward provided only at the end. We retain the same environment setting and optimization objective, and change where the reward is read.

\paragraph{Sources of intermediate signals.}
Intermediate signals can come from learned reward models \citep{agentprm,prints,istar,sparl,pair} or a post-hoc critic \citep{hcapo}, from the policy's own belief updates, as in the information-gain rewards of IGPO \citep{igpo}, from checklists scored by a language model, as in CM2 \citep{cm2}, from reference trajectories synthesized by language models, as in ToolVerse/TARA \citep{toolverse}, or from verifiers: VPR \citep{vpr} verifies individual intermediate actions. ProCredit belongs to the last category but reads the verifier differently: it reruns the full acceptance checks after every turn, defines progress as the fraction passed, and derives credit from the observed change between consecutive turns, without annotations, reward models, or reference trajectories. Because the signal is read from the environment state, a query that changes nothing earns nothing, whereas an information-gain reward pays for it.

\paragraph{Credit assignment in group-based reinforcement learning.}
GRPO \citep{grpo}, RLOO \citep{rloo}, and DAPO \citep{dapo} assign the same advantage to all turns of a trajectory. Advantages are zero in all-fail groups; DAPO addresses this through resampling, while ProGPO \citep{progpo} uses first-visit coverage as a fallback. Turn-level credit can come from hand-designed turn rewards, as in MT-GRPO \citep{mtgrpo}, which must be redesigned for each environment, or from a value function, as in Turn-PPO \citep{turnppo}, which like a reward model is a learned estimate. Another route refines the unit of comparison: GiGPO \citep{gigpo} and HGPO \citep{hgpo} group identical environment states across trajectories to compute step-level advantages, and G2PO and SALT propagate credit through state-transition graphs. These methods redistribute outcome rewards, and their step comparisons depend on states recurring across trajectories. ProCredit obtains turn rewards from the task's acceptance checks, which order the turns within a trajectory without coincident states across trajectories, and the same signal makes trajectories in all-fail groups comparable by completion.

\section{Method: ProCredit}
\label{sec:method}

ProCredit checks the environment state after each tool-execution turn during training, converts changes in progress into turn rewards, and uses these rewards at two levels: the total trajectory reward compares attempts at the same task, and returns from individual turns onward assign credit within each trajectory. The policy update uses the sum of the two advantages, as shown in Figure~\ref{fig:method}.

\begin{figure}[t]
    \centering
    \includegraphics[width=0.95\linewidth]{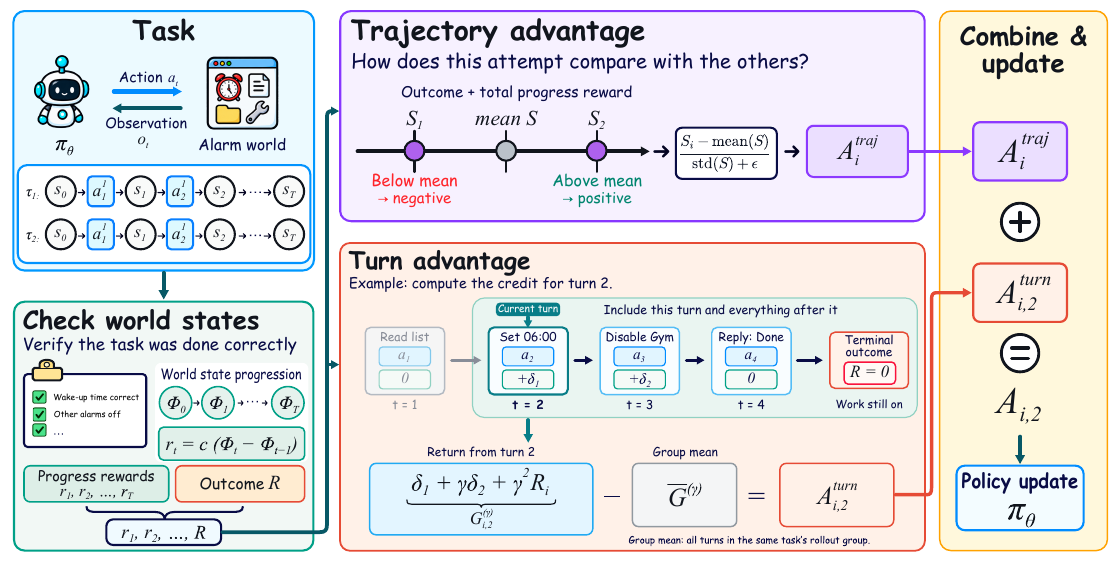}
    \caption{Overview of ProCredit. Each turn's change in progress becomes its reward; the rewards enter a trajectory-level advantage that compares attempts at the same task and a turn-level advantage that centers each turn's onward return. The lower panel traces Attempt B from Figure~\ref{fig:problem}: the turn that changes the time receives exactly $\delta_1$ more credit than the turn that disables Gym.}
    \label{fig:method}
\end{figure}

\subsection{Preliminaries}
\label{sec:preliminaries}

A task consists of repeated turns of action generation, tool execution, and observation. Let $g$ denote the task and $h_t$ the history visible to the model before turn $t$. The policy samples an action $a_t\sim\pi_\theta(\cdot\mid h_t,g)$, after which the environment enters state $s_t$ and returns observation $o_t$. The model receives observations rather than the full state, so these tasks are typically modeled as partially observable Markov decision processes \citep{loop,agenticrlsurvey}; the state $s_t$ is used for verification only during training. An attempt ends when the model declares completion or reaches a limit of $h$ turns; its actual length is $T$, and the environment then assigns an outcome reward $R\in\{0,1\}$. We denote the group of trajectories sampled for task $g$ by $\mathcal G_g$, index trajectories by $i$ and other members of the group by $j$, and omit $i$ when discussing a single trajectory.

Group relative policy optimization (GRPO) standardizes outcome rewards within the group to obtain a trajectory advantage:
\begin{equation}
    A_i = \frac{R_i-\operatorname{mean}_{j\in\mathcal G_g} R_j}
    {\operatorname{std}_{j\in\mathcal G_g} R_j+\epsilon}.
    \label{eq:grpo-advantage}
\end{equation}
Let $a_{i,t,k}$ be token $k$ of the action at turn $t$ in trajectory $i$, with importance ratio $\rho_{i,t,k}=\pi_\theta(a_{i,t,k}\mid h_{i,t},a_{i,t,<k},g)/\pi_{\theta_{\mathrm{old}}}(a_{i,t,k}\mid h_{i,t},a_{i,t,<k},g)$. The policy is updated with a clipped objective \citep{ppo}, using advantage $A_{i,t}$ at turn $t$:
\begin{equation}
    \mathcal J(\theta) = \mathbb E\Bigl[
    \tfrac{1}{\sum_{i,t}|a_{i,t}|}
    \textstyle\sum_{i,t}\sum_{k=1}^{|a_{i,t}|}
    \min\bigl(\rho_{i,t,k}A_{i,t},\
    \operatorname{clip}(\rho_{i,t,k},1-\varepsilon,1+\varepsilon)A_{i,t}\bigr)
    \Bigr].
\label{eq:policy-objective}
\end{equation}
GRPO uses the same advantage for every turn of a trajectory, $A_{i,t}=A_i$. Following LOOP \citep{loop}, we omit a KL term; the loss is averaged over all generated tokens and excludes tool responses.

\subsection{Verifiable task progress}
\label{sec:progress}

Task acceptance consists of executable checks, each testing a concrete fact about the environment. Let $\mathcal C_g$ be the check set, with size $K_g>0$, and let $C_{g,m}(s)\in\{0,1\}$ equal one if check $m$ passes in state $s$. We define task progress as the fraction of checks passed:
\begin{equation}
    \Phi_g(s) = \frac{1}{K_g}\sum_{m=1}^{K_g} C_{g,m}(s).
    \label{eq:progress}
\end{equation}
Progress after turn $t$ is $\Phi_t=\Phi_g(s_t)\in[0,1]$, and $\Phi_0$ is progress before any action. In our environments, we include only checks that fail in the initial state, so $\Phi_0=0$.

Progress depends on the environment state, rather than whether a tool call returns successfully: a query leaves progress unchanged, a modification raises it only if more acceptance checks pass, and an incorrect modification can invalidate previously passed checks and reduce it. In Figure~\ref{fig:problem}, reading the list leaves progress at zero, changing the time and disabling Gym each pass one check, and Work remains enabled, so the outcome reward $R$, recorded separately from progress, is zero. Progress is read only during training.

\subsection{Progress rewards and two levels of credit}
\label{sec:credit}

\paragraph{Progress rewards.}
The reward at turn $t$ is the change in progress between consecutive states, scaled by $c>0$:
\begin{equation}
    r_t = c\,(\Phi_t-\Phi_{t-1}).
    \label{eq:progress-reward}
\end{equation}
Turns that advance the task receive positive reward, turns that undo progress negative reward, and turns that leave the checks unchanged, including queries, zero. The trajectory score combines the outcome reward with the sum of progress rewards:
\begin{equation}
    S = R+\sum_{t=1}^{T}r_t = R+c\,(\Phi_T-\Phi_0).
    \label{eq:trajectory-score}
\end{equation}
The score depends only on initial and final progress, as in potential-based reward shaping \citep{ng1999} with potential $c\,\Phi_g(s)$. Unlike \citet{ng1999}, we do not set terminal potentials to zero, so the score retains final completion. Successful trajectories still score strictly higher than any failed trajectory, and the set of optimal policies is unchanged when the task is solvable (Appendix~\ref{app:score-properties}), so progress refines the ranking only among attempts with the same outcome; the turn-level term below uses when that progress was made.

\paragraph{Trajectory-level advantage.}
We replace $R_i$ in GRPO's within-group standardization with the trajectory score $S_i$:
\begin{equation}
    A_i^{\mathrm{traj}} =
    \frac{S_i-\operatorname{mean}_{j\in\mathcal G_g} S_j}
    {\operatorname{std}_{j\in\mathcal G_g} S_j+\epsilon}.
    \label{eq:trajectory-advantage}
\end{equation}
This term treats all turns of a trajectory alike and expresses how the attempt compares with the other attempts at the same task.

\paragraph{Turn-level advantage.}
Credit at turn $t$ uses the discounted return from that turn onward:
\begin{equation}
    G_{i,t}^{(\gamma)} = \gamma^{T_i-t}R_i
    +\sum_{u=t}^{T_i}\gamma^{u-t}r_{i,u}.
    \label{eq:discounted-return}
\end{equation}
ProCredit uses $\gamma=1$. An attempt has at most $h$ turns and bounded rewards, so its undiscounted return is finite and does not require discounting to converge. The return has the closed form
\begin{equation}
    G_{i,t} = R_i+c\,(\Phi_{i,T_i}-\Phi_{i,t-1}),
    \label{eq:undiscounted-return}
\end{equation}
which is the outcome reward plus the progress still to be earned from this turn on. We examine the effect of discounting in Section~\ref{sec:discount}. Let $\mathcal U_g$ contain all turns of all trajectories for task $g$. We center turn returns by their mean over $\mathcal U_g$, without dividing by a standard deviation:
\begin{equation}
    A_{i,t}^{\mathrm{turn}} = G_{i,t}
    -\operatorname{mean}_{(j,u)\in\mathcal U_g}G_{j,u}.
    \label{eq:turn-advantage}
\end{equation}
Appendix~\ref{app:pooling} explains why turns with different histories can share this reference value. The final advantage is the sum of the two terms:
\begin{equation}
    A_{i,t} = A_i^{\mathrm{traj}}+A_{i,t}^{\mathrm{turn}}.
    \label{eq:combined-advantage}
\end{equation}

\paragraph{Roles of the two terms.}
The trajectory-level term determines whether an attempt is reinforced or suppressed as a whole; the turn-level term adds a tilt within the trajectory. When $\gamma=1$, the difference in credit between adjacent turns is exactly the progress reward of the earlier turn:
\begin{equation}
    A_{i,t}-A_{i,t+1} = r_{i,t}.
    \label{eq:adjacent-credit}
\end{equation}
A turn that advances the task receives more credit than the next turn by exactly the reward for its own progress, and a query that leaves the state unchanged receives the same credit as the next turn, as in the lower panel of Figure~\ref{fig:method}. This tilt depends only on progress differences, not on success or failure, so turns can still receive different credit in groups where every trajectory succeeds (Appendix~\ref{app:credit-example} gives a complete numerical example). The policy is then updated with the objective in Equation~\eqref{eq:policy-objective}, using $A_{i,t}$ in place of the shared $A_i$; all tokens generated within a turn share that turn's advantage.

\section{Experiments}
\label{sec:experiments}

\begin{table}[tbp]
\centering
\caption{TGC and SGC (\%) on AppWorld Test-Normal and Test-Challenge.}
\label{tab:main}
\normalsize
\renewcommand{\arraystretch}{1.08}
\setlength{\tabcolsep}{3pt}
\begin{tabular*}{\linewidth}{@{\extracolsep{\fill}}lcccccc@{}}
\toprule
& \multicolumn{6}{c}{\textbf{Test-Normal (N)}} \\
& \multicolumn{2}{c}{\textbf{4B}} & \multicolumn{2}{c}{\textbf{9B}} & \multicolumn{2}{c}{\textbf{35B}} \\
\cmidrule(lr){2-3}\cmidrule(lr){4-5}\cmidrule(lr){6-7}
\textbf{Method} & \textbf{TGC $\uparrow$} & \textbf{SGC $\uparrow$} & \textbf{TGC $\uparrow$} & \textbf{SGC $\uparrow$} & \textbf{TGC $\uparrow$} & \textbf{SGC $\uparrow$} \\
\midrule
Base model & \scoreonly{17.78} & \scoreonly{4.24} & \scoreonly{19.49} & \scoreonly{4.91} & \scoreonly{37.28} & \scoreonly{12.28} \\
SFT & \scoreonly{19.87} & \scoreonly{6.47} & \scoreonly{31.70} & \scoreonly{6.70} & \scoreonly{49.78} & \scoreonly{21.43} \\
\midrule
\rowcolor{tablegroup}
\multicolumn{7}{c}{\emph{Outcome rewards only}} \\
GRPO & \score{66.32}{1.27} & \score{42.34}{3.17} & \score{72.79}{0.96} & \score{51.64}{1.95} & \score{75.00}{0.61} & \score{57.37}{0.85} \\
RLOO & \score{65.87}{3.76} & \score{42.04}{5.36} & \score{72.72}{3.10} & \score{51.26}{4.22} & \secondscore{78.45}{0.76} & \score{58.53}{2.59} \\
DAPO & \secondscore{67.46}{5.62} & \secondscore{43.38}{6.19} & \secondscore{76.96}{4.53} & \secondscore{56.70}{3.98} & \score{77.90}{1.24} & \bestscore{61.23}{2.46} \\
SALT & \score{65.60}{3.92} & \score{40.25}{6.26} & \score{74.68}{2.42} & \score{54.61}{4.22} & \score{78.27}{2.45} & \score{58.31}{3.38} \\
\midrule
\rowcolor{tablegroup}
\multicolumn{7}{c}{\emph{Outcome rewards + progress}} \\
GRPO-$\Phi$ & \score{65.40}{4.12} & \score{41.67}{6.25} & \score{74.78}{2.54} & \score{54.24}{4.68} & \score{76.41}{1.40} & \score{58.48}{2.70} \\
GiGPO & \score{64.53}{3.78} & \score{42.26}{6.56} & \score{69.77}{2.21} & \score{47.25}{3.21} & \score{56.99}{3.28} & \score{35.49}{2.88} \\
\textbf{ProCredit} & \bestscore{71.53}{0.48} & \bestscore{48.88}{0.45} & \bestscore{78.15}{2.61} & \bestscore{59.82}{3.95} & \bestscore{78.65}{0.85} & \secondscore{59.60}{1.35} \\
\midrule
& \multicolumn{6}{c}{\textbf{Test-Challenge (C)}} \\
& \multicolumn{2}{c}{\textbf{4B}} & \multicolumn{2}{c}{\textbf{9B}} & \multicolumn{2}{c}{\textbf{35B}} \\
\cmidrule(lr){2-3}\cmidrule(lr){4-5}\cmidrule(lr){6-7}
\textbf{Method} & \textbf{TGC $\uparrow$} & \textbf{SGC $\uparrow$} & \textbf{TGC $\uparrow$} & \textbf{SGC $\uparrow$} & \textbf{TGC $\uparrow$} & \textbf{SGC $\uparrow$} \\
\midrule
Base model & \scoreonly{10.67} & \scoreonly{1.80} & \scoreonly{12.59} & \scoreonly{3.78} & \scoreonly{27.82} & \scoreonly{8.81} \\
SFT & \scoreonly{9.95} & \scoreonly{2.52} & \scoreonly{20.14} & \scoreonly{5.76} & \scoreonly{31.59} & \scoreonly{8.63} \\
\midrule
\rowcolor{tablegroup}
\multicolumn{7}{c}{\emph{Outcome rewards only}} \\
GRPO & \score{46.50}{0.66} & \score{24.46}{0.78} & \score{60.77}{2.35} & \score{37.23}{2.03} & \secondscore{66.73}{3.66} & \secondscore{43.35}{2.96} \\
RLOO & \score{48.28}{4.55} & \score{26.20}{3.06} & \score{60.13}{4.87} & \score{35.61}{6.08} & \score{63.09}{0.52} & \score{41.49}{2.37} \\
DAPO & \secondscore{49.06}{7.43} & \secondscore{27.94}{7.11} & \secondscore{63.13}{2.39} & \secondscore{39.93}{2.12} & \score{64.31}{1.58} & \score{42.63}{2.59} \\
SALT & \score{48.50}{3.28} & \score{25.66}{2.22} & \score{58.73}{2.52} & \score{35.73}{3.40} & \score{65.21}{4.31} & \score{43.11}{4.28} \\
\midrule
\rowcolor{tablegroup}
\multicolumn{7}{c}{\emph{Outcome rewards + progress}} \\
GRPO-$\Phi$ & \score{47.82}{2.38} & \score{25.48}{2.44} & \score{62.15}{2.99} & \score{38.85}{2.77} & \score{65.05}{1.03} & \secondscore{43.35}{1.65} \\
GiGPO & \score{43.69}{4.33} & \score{22.96}{3.54} & \score{57.89}{2.43} & \score{32.67}{3.10} & \score{41.43}{3.27} & \score{17.99}{3.53} \\
\textbf{ProCredit} & \bestscore{52.96}{0.50} & \bestscore{29.26}{0.45} & \bestscore{63.85}{0.45} & \bestscore{41.07}{0.73} & \bestscore{68.11}{3.95} & \bestscore{45.68}{4.40} \\
\bottomrule
\end{tabular*}
\end{table}

\subsection{Experimental setup}
\label{sec:setup}

\paragraph{Environments.}
AppWorld requires agents to complete tasks across applications through multiple turns of tool use, with executable unit tests determining task completion. We train on 90 tasks, use a development set of 57 tasks, and evaluate on test-normal (168 tasks) and test-challenge (417 tasks); the latter includes applications unseen during training. We report task goal completion (TGC) and scenario goal completion (SGC); a scenario is complete only when all its tasks are complete. ToolSandbox~\citep{toolsandbox} is a conversational tool environment with state-modifying tools and milestone-based evaluation. Of its 101 tasks, we use 52 for training and 21 for evaluation; the remaining tasks are excluded because the base model completes all of them or they have fewer than two milestones.

\paragraph{Models and training.}
We use Qwen3.5~\citep{qwen35} models: Qwen3.5-4B, Qwen3.5-9B, and Qwen3.5-35B-A3B (abbreviated as 35B). Reinforcement learning starts from the base models, and comparisons are made within each model size. ProCredit uses $c=0.5$ and $\gamma=1$. For each reinforcement learning method, we report the mean and standard deviation over three independent training runs.

\paragraph{Comparison methods.}
SFT uses trajectories generated on AppWorld. Outcome-only methods include GRPO~\citep{grpo}, RLOO~\citep{rloo}, DAPO~\citep{dapo}, and SALT~\citep{salt}. Methods that use progress include GRPO-$\Phi$, GiGPO~\citep{gigpo}, and ProCredit. GRPO-$\Phi$ adds final progress to the trajectory score, giving the same score as ProCredit, but uses only trajectory-level advantages. GiGPO uses the same progress rewards as ProCredit and groups states by the set of passed checks; the main table uses its default discount factor, $\gamma=0.95$. SALT also defines state equivalence classes by the set of passed checks.

\subsection{Main results}
\label{sec:main-results}

\paragraph{Overall performance.}
Table~\ref{tab:main} groups methods by their reward source and reports results on both test sets. ProCredit achieves the highest TGC at all three model sizes on both test sets, and the highest SGC in every column except 35B test-normal, where DAPO is slightly higher. Its margin over the strongest outcome-only method is largest at 4B, 4.1 and 3.9 TGC points on test-normal and test-challenge with the smallest standard deviations in those columns, and narrows to 1 to 3 points at 9B; at 35B, RLOO, DAPO, SALT, and ProCredit all score around 78 in test-normal TGC, and ProCredit leads by 1 to 2 points on test-challenge. GiGPO uses the same progress rewards yet scores below GRPO at 9B and 35B; Section~\ref{sec:discount} traces this to its default discount factor.

\paragraph{Same progress, different credit.}
GRPO-$\Phi$ and ProCredit assign the same score to a given trajectory and differ only in the turn-level term. ProCredit scores higher at all three model sizes on both test sets: its test-normal TGC is higher by 6.1, 3.4, and 2.2 points at 4B, 9B, and 35B, and its test-challenge TGC by 5.1, 1.7, and 3.1 points. GRPO-$\Phi$ itself does not consistently outperform GRPO. Adding progress to the trajectory score alone is insufficient; the difference comes from when progress is credited. Appendix~\ref{app:learning} plots the training and development curves of the three methods.

\begin{table}[htbp]
\centering
\caption{Success rates (\%) on ToolSandbox.}
\label{tab:toolsandbox}
\small
\setlength{\tabcolsep}{10pt}
\begin{tabular}{lccc}
\toprule
\textbf{Model} & \textbf{Base model} & \textbf{GRPO} & \textbf{ProCredit} \\
\midrule
4B & \scoreonly{63.1} & \score{81.0}{1.0} & \bestscore{84.5}{0.5} \\
9B & \scoreonly{72.0} & \score{85.7}{1.8} & \bestscore{89.3}{1.9} \\
35B & \scoreonly{79.8} & \score{89.9}{2.6} & \bestscore{92.3}{2.9} \\
\bottomrule
\end{tabular}
\end{table}

\paragraph{A second environment.}
Table~\ref{tab:toolsandbox} reports success rates on ToolSandbox. ProCredit outperforms GRPO at each model size. On the single-turn and multi-turn function-calling categories evaluated in BFCL~\citep{bfcl}, ProCredit's scores after training remain close to those of the base models; Appendix~\ref{app:bfcl} reports the full results.

\subsection{Core ablations}
\label{sec:ablations}

Table~\ref{tab:granularity} reports TGC and SGC on test-normal: the upper block varies the credit unit with undiscounted returns, and the lower block discounts returns with $\gamma=0.95$. Rows with the same configuration as the main table reuse its results.

\begin{table}[tbp]
\centering
\caption{Ablations on Test-Normal (\%): the same progress signal credited at different granularities, with undiscounted returns (upper block) and with $\gamma=0.95$ (lower block).}
\label{tab:granularity}
\small
\setlength{\tabcolsep}{3pt}
\begin{tabular*}{\linewidth}{@{\extracolsep{\fill}}llcccccc@{}}
\toprule
& & \multicolumn{2}{c}{\textbf{4B}} & \multicolumn{2}{c}{\textbf{9B}} & \multicolumn{2}{c}{\textbf{35B}} \\
\cmidrule(lr){3-4}\cmidrule(lr){5-6}\cmidrule(lr){7-8}
\textbf{Method} & \textbf{Credit unit} & \textbf{TGC $\uparrow$} & \textbf{SGC $\uparrow$} & \textbf{TGC $\uparrow$} & \textbf{SGC $\uparrow$} & \textbf{TGC $\uparrow$} & \textbf{SGC $\uparrow$} \\
\midrule
\rowcolor{tablegroup}
\multicolumn{8}{c}{\emph{Credit granularity, undiscounted returns}} \\
GRPO & Trajectory & \score{66.32}{1.27} & \score{42.34}{3.17} & \score{72.79}{0.96} & \score{51.64}{1.95} & \score{75.00}{0.61} & \score{57.37}{0.85} \\
GRPO-$\Phi$ & Trajectory & \score{65.40}{4.12} & \score{41.67}{6.25} & \score{74.78}{2.54} & \score{54.24}{4.68} & \score{76.41}{1.40} & \score{58.48}{2.70} \\
GiGPO & Milestone & \score{64.66}{5.09} & \score{39.73}{6.61} & \secondscore{75.30}{1.95} & \secondscore{54.84}{3.23} & \score{74.40}{5.17} & \score{56.25}{6.22} \\
ProCredit (token) & Token & \score{63.19}{3.69} & \score{38.54}{4.72} & \score{74.60}{4.84} & \score{51.86}{8.97} & \secondscore{78.47}{3.59} & \bestscore{61.61}{5.03} \\
\textbf{ProCredit (turn)} & Turn & \bestscore{71.53}{0.48} & \bestscore{48.88}{0.45} & \bestscore{78.15}{2.61} & \bestscore{59.82}{3.95} & \bestscore{78.65}{0.85} & \secondscore{59.60}{1.35} \\
\midrule
\rowcolor{tablegroup}
\multicolumn{8}{c}{\emph{Discounted returns, $\gamma=0.95$}} \\
GiGPO & Milestone & \score{64.53}{3.78} & \score{42.26}{6.56} & \score{69.77}{2.21} & \score{47.25}{3.21} & \score{56.99}{3.28} & \score{35.49}{2.88} \\
ProCredit (turn) & Turn & \secondscore{68.48}{2.53} & \secondscore{45.24}{5.61} & \score{68.20}{5.66} & \score{45.98}{8.65} & \score{65.13}{7.37} & \score{44.64}{8.12} \\
\bottomrule
\end{tabular*}
\end{table}

\subsubsection{Rewards and credit assignment}
\label{sec:granularity}

\paragraph{Credit granularity.}
The upper block of Table~\ref{tab:granularity} assigns the same progress signal at different granularities. GRPO and GRPO-$\Phi$ credit entire trajectories, without and with final completion in the score. GiGPO with $\gamma=1$ groups turns by state anchors and assigns credit at milestones. ProCredit (turn) assigns credit per turn. ProCredit (token) distributes each turn's progress reward linearly across token positions: for token $k$ among the turn's $n$ tokens, the return includes only the remaining portion of that turn's progress reward, giving $G_{i,t}-r_{i,t}\,k/n$, with grouping and centering unchanged. This variant introduces no token-level environmental feedback; it only dilutes the same reward.

\paragraph{Turn-level credit works best.}
ProCredit (turn) exceeds ProCredit (token) in TGC by 8.3 points at 4B and 3.6 points at 9B, with larger gaps in SGC. The token variant falls below GRPO at 4B, exceeds it at 9B with much larger variation across runs, and matches the turn variant in TGC at 35B, where its slightly higher mean SGC lies within its own variability. GiGPO with $\gamma=1$ uses the same progress rewards but scores below ProCredit (turn) at all three model sizes; their turn-level advantages differ only by an anchor-group mean term, as shown in Appendix~\ref{app:pooling}.

\subsubsection{Discount factor}
\label{sec:discount}

\paragraph{The discount factor has the largest individual effect.}
GiGPO adopts $\gamma=0.95$ without an ablation, whereas GEM~\citep{gem} and a practical guide to multi-turn agent reinforcement learning~\citep{agenticrlguide} argue that $\gamma<1$ is incompatible with group-based multi-turn training. In turn-level credit assignment, $\gamma$ sets how much credit for early turns decays with each subsequent turn, and hence the model's preference over trajectory length. The lower block of Table~\ref{tab:granularity} discounts both ProCredit and GiGPO. Changing ProCredit from $\gamma=0.95$ to $\gamma=1$ increases TGC by approximately 3, 10, and 13 points at 4B, 9B, and 35B; the gap grows with model size. GiGPO also scores higher without discounting at 9B and 35B, which accounts for most of its deficit in Table~\ref{tab:main}, although it remains below ProCredit at all three model sizes. Appendix~\ref{app:group-variants} examines discounting only selected types of sampled groups.

\paragraph{Why discounting reduces performance.}
With $\gamma<1$, the outcome-reward component of a turn's credit is $\gamma^{T-t}R$: a turn at the same position receives strictly less credit in a longer trajectory, independently of the task. Within an all-success group for the same task, the shortest trajectory therefore almost always receives the highest credit, and the model learns to use fewer turns. After training, ProCredit with $\gamma=0.95$ produces trajectories that are 9\% and 20\% shorter on average than GRPO at 4B and 9B, whereas $\gamma=1$ yields lengths close to GRPO. The shorter the trajectories are pushed, the more accuracy is lost, and the losses concentrate on medium and hard tasks, whose required steps discounting removes. Discounting thus adds a \emph{length tax}, a reduction in earlier credit caused by additional turns, and ProCredit uses $\gamma=1$. Appendix~\ref{app:length-tax} states this tax as a proposition and reports credit statistics from actual training trajectories and results by task difficulty.

\subsection{Analysis}
\label{sec:analysis}

\paragraph{Progress differences among failed attempts.}
We call a sampled group all-success, mixed, or all-fail according to whether all, some, or none of its trajectories succeed. In ProCredit's training trajectories, 54.9\% and 38.6\% of all-fail groups at 4B and 9B, respectively, contain distinguishable progress differences. These groups provide no gradient under outcome-only rewards but still provide learning signals under ProCredit. Figure~\ref{fig:failed-progress} compares the final progress of failed trajectories under GRPO, GRPO-$\Phi$, and ProCredit in two group types. In mixed groups, both progress-based methods produce failed trajectories with higher completion than GRPO early in training; GRPO-$\Phi$ declines during the middle and later stages, while ProCredit remains higher. In all-fail groups, both progress-based methods exceed GRPO later in training and differ little from each other.

\begin{figure}[t]
\centering
\includegraphics[width=0.9\linewidth]{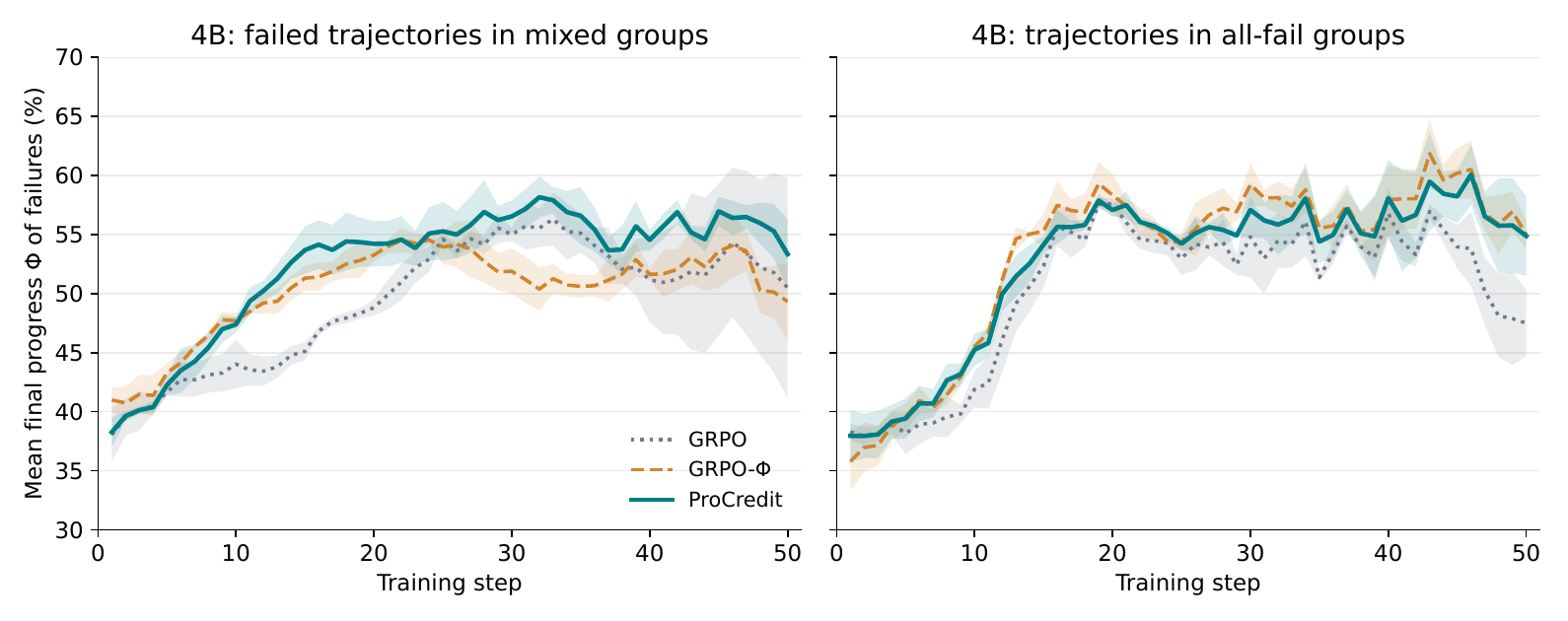}
\caption{Final progress $\Phi_T$ of failed trajectories during 4B training, in mixed groups (left) and all-fail groups (right).}
\label{fig:failed-progress}
\end{figure}

\paragraph{Output entropy after tool feedback.}
The trained models respond more distinctly to progress and error feedback than GRPO: output entropy is lower after turns that advance the task, while the entropy increase after errors is preserved. Appendix~\ref{app:entropy} provides the entropy curves, computation details, and results at all three model sizes.

\section{Discussion}
\label{sec:discussion}

ProCredit requires task progress to be measurable at intermediate states, which holds whenever the environment state is accessible and the acceptance checks can be rerun. AppWorld's unit tests and ToolSandbox's milestones provide examples; test cases for coding tasks and tasks evaluated by their final database state also fit this setting, and other environments can use an intermediate-state verifier. The checks also set the granularity of progress: more checks passed across more turns let turn-level credit distinguish more turns, whereas with one or two checks, or all checks passed on the final turn, the turn-level term is constant within a trajectory and the method reduces to a GRPO-style comparison of trajectory scores.

Progress must come from the environment rather than the policy's own predictions: when a model predicts its own dense signal and is rewarded for it, group normalization lets the policy improve that prediction without increasing task success~\citep{darkroom}. ProCredit's progress is determined by the environment state, so the policy can obtain it only by changing that state.

The gains narrow at larger model sizes, consistent with the information supplied by progress: larger models produce fewer all-fail groups and failed attempts, leaving less missing information for progress to recover. In harder environments, where mixed and all-fail groups are as common for a larger model as they are for the 4B model here, progress has correspondingly more to supply. Our current progress measure is the fraction of passed checks, treating checks as equally weighted and unordered; future work will extend ProCredit to larger models and harder environments, and adapt progress measurement to those where a verifier decomposes final acceptance into intermediate checks.

\section{Conclusion}
\label{sec:conclusion}

Outcome rewards reduce a long-horizon attempt to one number and discard how much progress was made and when. ProCredit runs the acceptance checks on intermediate states during training, reads progress from the environment state, and uses its change to assign trajectory-level and turn-level credit. On AppWorld it achieves the highest task completion rate at three model sizes on both test sets, including applications unseen in training. The benefit comes from crediting progress to the turn where it occurs: adding progress only to trajectory scores or moving to token-level credit does not improve performance, while discounting credit for early turns reduces it.

\FloatBarrier
\bibliography{references}
\bibliographystyle{plainnat}

\clearpage
\appendix
\raggedbottom
\let\appendixsectionorig\section
\renewcommand{\section}{\FloatBarrier\appendixsectionorig}
\renewcommand{\thetable}{\Alph{section}\arabic{table}}
\renewcommand{\thefigure}{\Alph{section}\arabic{figure}}
\renewcommand{\theHtable}{\Alph{section}.\arabic{table}}
\renewcommand{\theHfigure}{\Alph{section}.\arabic{figure}}

\section{Output entropy after tool feedback}
\label{app:entropy}
\setcounter{table}{0}
\setcounter{figure}{0}

\paragraph{Computation and comparison.}
We compute full-vocabulary Shannon entropy of the model's output after tool feedback at temperature 1, in nats, over the first 120 tokens. We classify feedback into tool errors, progress without errors, and unchanged progress without errors; we call the last category ordinary turns. Errors take precedence in classification. First turns, progress regressions, and records with unknown status are excluded from ordinary turns. We retain only untruncated turns with at least 120 output tokens. We average curves within trajectories, within tasks, and across tasks, then give equal weight to the three seeds. Shading shows the sample standard deviation across seeds. For paired differences, we first subtract ordinary-turn entropy from feedback-turn entropy within the same trajectory, then aggregate over the beginning (tokens 1 to 20), middle (21 to 60), and end (61 to 120). This statistic differs from subtracting the two overall mean curves directly.

\paragraph{Overall entropy and feedback differences at 4B.}
In Figure~\ref{fig:entropy-4b}, both training methods reduce ordinary-turn entropy in the middle segment, from 0.57 before training to 0.24 for GRPO and 0.37 for ProCredit. Meanwhile, the middle-segment entropy difference between progress and ordinary turns is $-0.093\pm0.014$ for ProCredit and $-0.025\pm0.019$ for GRPO. All three ProCredit seeds have lower values than all three GRPO seeds. The beginning-segment difference between error and ordinary turns is $+0.235\pm0.034$ before training, $+0.245\pm0.027$ for ProCredit, and $+0.129\pm0.068$ for GRPO. ProCredit preserves the entropy increase after errors and widens the middle-segment difference between progress and ordinary turns.

\begin{figure}[!htbp]
\centering
\includegraphics[width=0.7\linewidth]{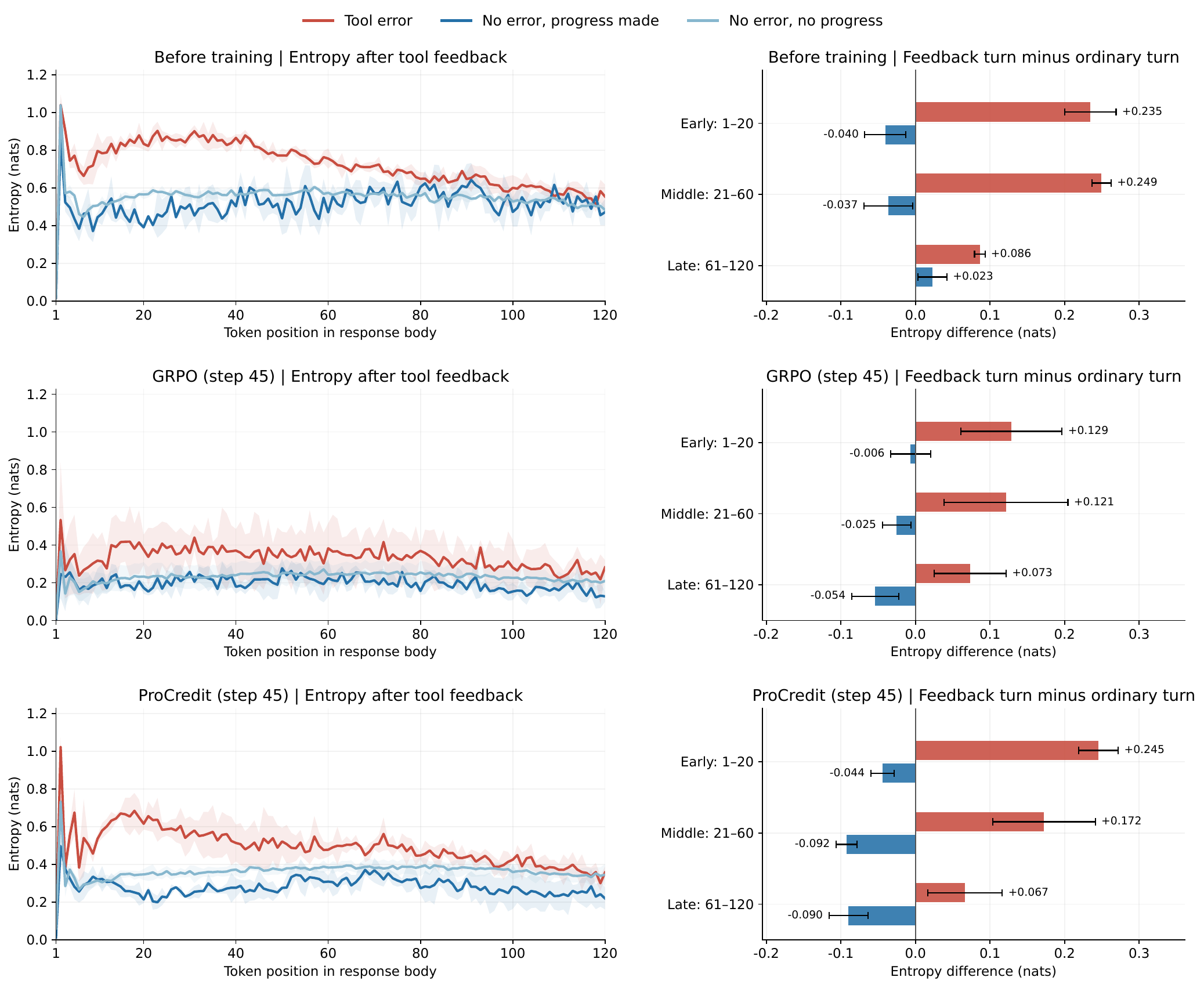}
\caption{Output entropy after tool feedback for 4B before training, after GRPO, and after ProCredit.}
\label{fig:entropy-4b}
\end{figure}

\paragraph{Feedback differences at 9B depend on output position.}
In Figure~\ref{fig:entropy-9b}, ProCredit's error curve starts above the ordinary-turn curve and gradually approaches it; its progress curve is lower in the middle segment. Within-trajectory paired results show the same change. The middle-segment progress-minus-ordinary difference is $-0.062\pm0.011$ for ProCredit and $-0.037\pm0.041$ for GRPO. The end-segment error-minus-ordinary differences are $-0.008\pm0.021$ and $+0.065\pm0.040$, respectively. At 9B, the entropy increase after errors is therefore concentrated at earlier positions and approaches ordinary-turn levels toward the end.

\begin{figure}[!htbp]
\centering
\includegraphics[width=0.7\linewidth]{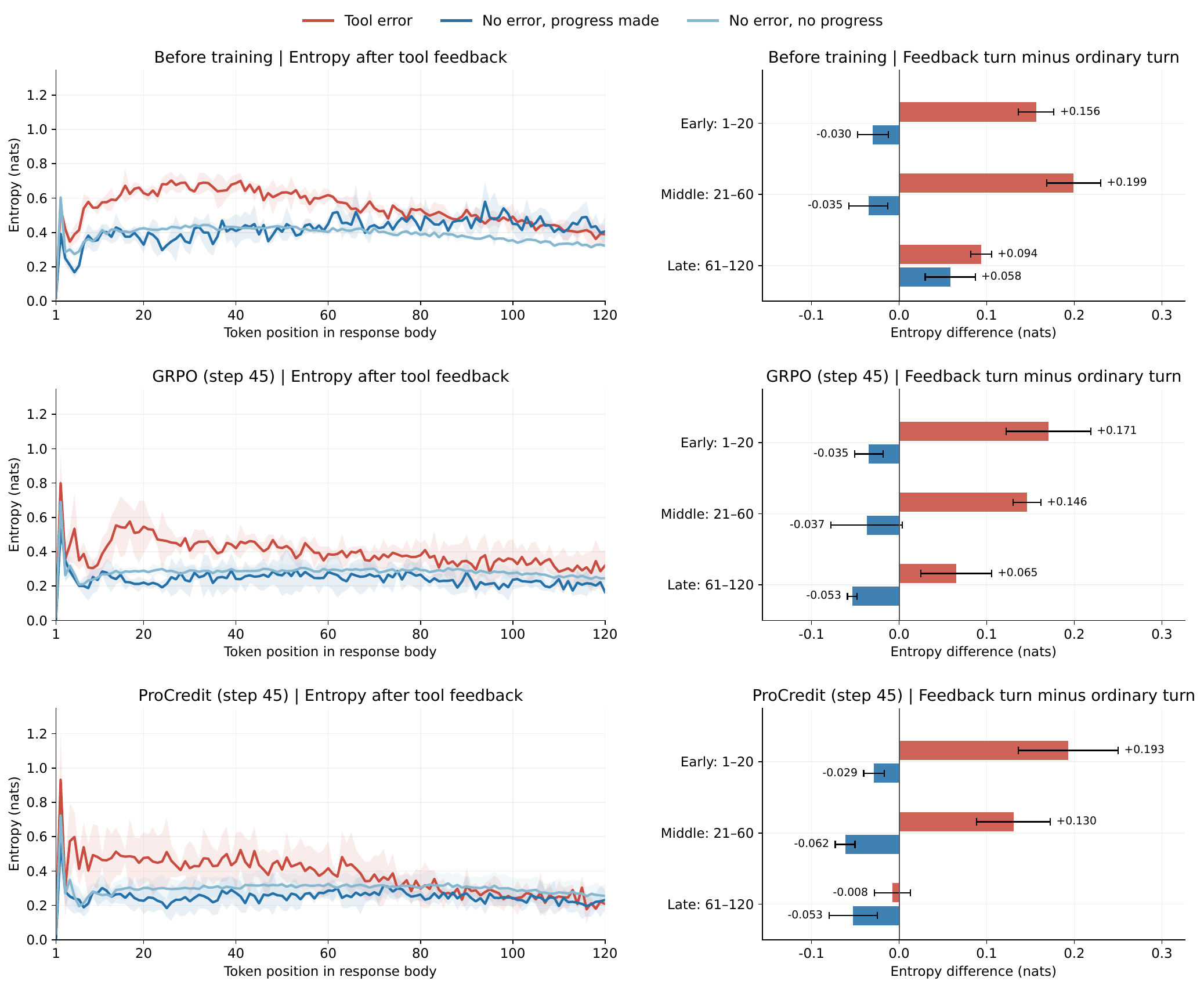}
\caption{Output entropy after tool feedback for 9B before training, after GRPO, and after ProCredit.}
\label{fig:entropy-9b}
\end{figure}

\paragraph{35B also distinguishes feedback types.}
In Figure~\ref{fig:entropy-35b}, ProCredit's progress curve lies below the ordinary-turn curve in the middle segment; entropy after errors rises early and then falls. The middle-segment progress-minus-ordinary difference changes from $-0.029\pm0.022$ before training to $-0.080\pm0.019$ under ProCredit. The end-segment error-minus-ordinary difference changes from $+0.080\pm0.021$ to $+0.017\pm0.010$. Output distributions continue to differ across feedback types, with the magnitude and duration of these differences depending on model size and training method.

\begin{figure}[!htbp]
\centering
\includegraphics[width=0.7\linewidth]{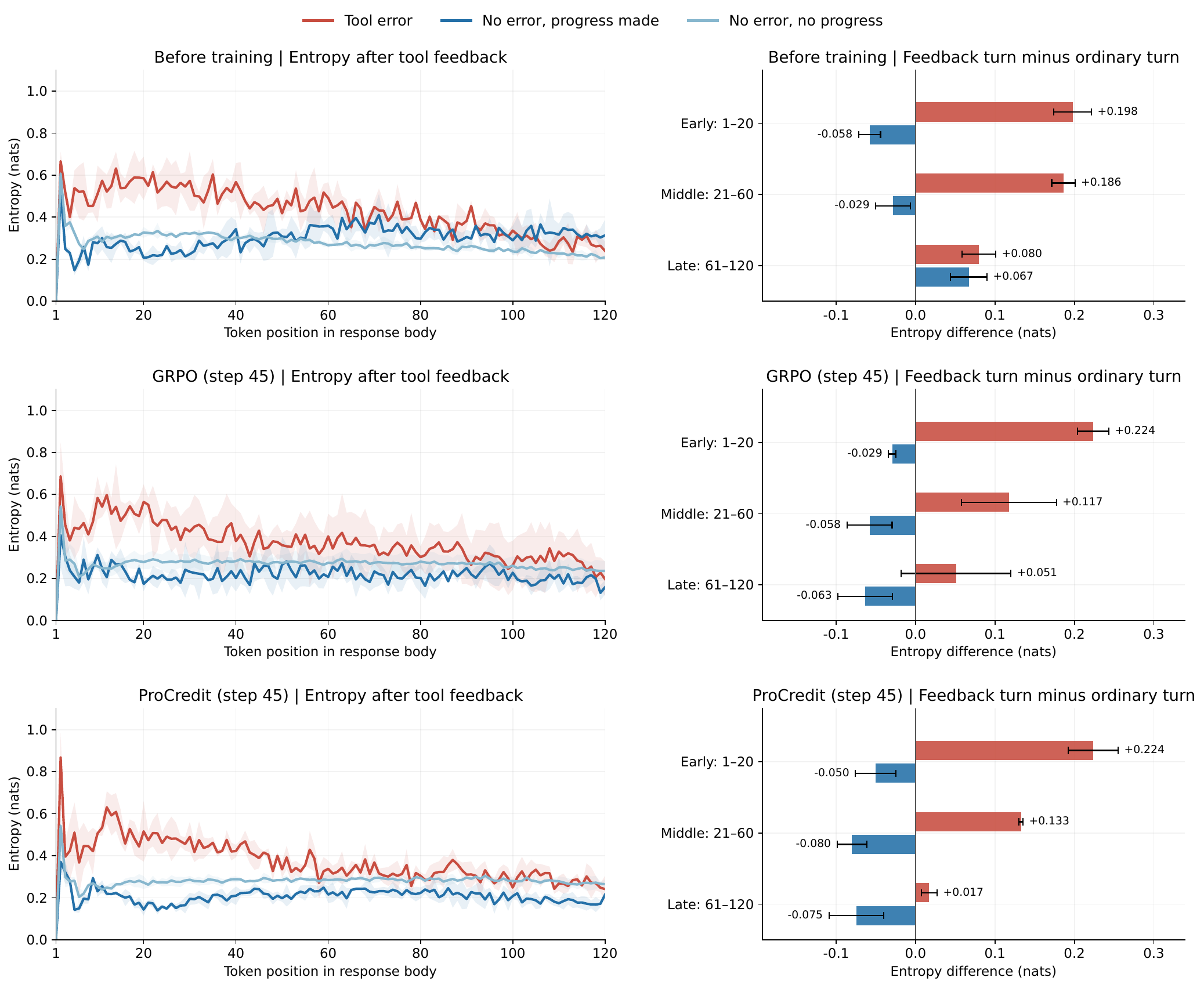}
\caption{Output entropy after tool feedback for 35B before training, after GRPO, and after ProCredit.}
\label{fig:entropy-35b}
\end{figure}

\section{A worked credit example for the alarm task}
\label{app:credit-example}
\setcounter{table}{0}
\setcounter{figure}{0}

We apply the computation in Section~\ref{sec:credit} to the task in Figure~\ref{fig:problem}. The three checks require setting the wake-up alarm to 06:00, disabling Gym, and disabling Work; $K_g=3$, $c=0.5$, and $\gamma=1$. The illustrative group contains three trajectories: A reads the alarm list and declares completion; B changes the time and disables Gym but misses Work; C completes all three requirements. Table~\ref{tab:alarm-example} lists each turn's progress $\Phi_t$, progress reward $r_t$, return $G_t$, and two levels of advantage.

The three trajectory scores $S$ are 0, 0.333, and 1.5, and their group-standardized trajectory-level advantages are $-0.776$, $-0.353$, and $+1.129$. Completion distinguishes A and B despite both failing. The mean return over all turns for the task is 0.667; each turn-level advantage is its return minus this mean.

\begin{table}[!htbp]
\centering
\caption{Per-turn progress, reward, return, and advantages for three illustrative trajectories on the task in Figure~\ref{fig:problem}.}
\label{tab:alarm-example}
\small
\setlength{\tabcolsep}{3pt}
\begin{adjustbox}{max width=\linewidth}
\begin{tabular}{lclccccc}
\toprule
\textbf{Trajectory} & \textbf{Turn} & \textbf{Action} & $\boldsymbol{\Phi_t}$ & $\boldsymbol{r_t}$ & $\boldsymbol{G_t}$ & \textbf{\shortstack{Turn-level\\advantage}} & \textbf{\shortstack{Total\\advantage}} \\
\midrule
A (failed) & $1$ & Read alarm list & $0$ & $0$ & $0$ & $-0.667$ & $-1.443$ \\
 & $2$ & Declare completion & $0$ & $0$ & $0$ & $-0.667$ & $-1.443$ \\
B (failed) & $1$ & Read alarm list & $0$ & $0$ & $0.333$ & $-0.333$ & $-0.686$ \\
 & $2$ & Set wake-up alarm to 06:00 & $1/3$ & $+0.167$ & $0.333$ & $-0.333$ & $-0.686$ \\
 & $3$ & Disable Gym & $2/3$ & $+0.167$ & $0.167$ & $-0.500$ & $-0.853$ \\
 & $4$ & Declare completion & $2/3$ & $0$ & $0$ & $-0.667$ & $-1.019$ \\
C (successful) & $1$ & Read alarm list & $0$ & $0$ & $1.5$ & $+0.833$ & $+1.962$ \\
 & $2$ & Set wake-up alarm to 06:00 & $1/3$ & $+0.167$ & $1.5$ & $+0.833$ & $+1.962$ \\
 & $3$ & Disable Gym & $2/3$ & $+0.167$ & $1.333$ & $+0.667$ & $+1.795$ \\
 & $4$ & Disable Work & $1$ & $+0.167$ & $1.167$ & $+0.500$ & $+1.629$ \\
 & $5$ & Declare completion & $1$ & $0$ & $1.0$ & $+0.333$ & $+1.462$ \\
\bottomrule
\end{tabular}
\end{adjustbox}
\end{table}

The table illustrates three properties from Section~\ref{sec:credit}. Reading the list and the immediately following progress turn receive equal credit because they have the same remaining progress to earn. After each progress turn, the next turn's credit decreases by exactly the progress reward just earned, 0.167. In B, declaring completion after making progress has the lowest turn-level credit, tied with both turns in A. Across trajectories, every turn in C has higher total credit than any turn in B, and every turn in B has credit at least as high as those in A.

\section{How discounting penalizes long trajectories}
\label{app:length-tax}
\setcounter{table}{0}
\setcounter{figure}{0}

Section~\ref{sec:discount} shows that discounting leads models to use fewer turns, with accuracy losses concentrated on medium and hard tasks. We first state where discounting reduces credit, then use illustrative and actual training trajectories to quantify the reduction, and finally examine the trained models by task difficulty.

\begin{proposition}[Length tax]
\label{prop:length-tax}
Using the notation of Section~\ref{sec:credit}, consider turn $p$ of a trajectory. For any $t\le p$, let the discounted value at turn $t$ of rewards received after $p$ be
\begin{equation}
D_t^{(p)}=\sum_{u=p+1}^{T}\gamma^{u-t}r_u+\gamma^{T-t}R.
\end{equation}
Inserting a state-preserving turn after turn $p$ delays every subsequent reward, including the outcome reward, by one turn. The return $G_t^{(\gamma)}$ then decreases by $(1-\gamma)D_t^{(p)}$ for every $t\le p$.
\end{proposition}

\begin{proof}
Before insertion, $G_t^{(\gamma)}=\sum_{u=t}^{p}\gamma^{u-t}r_u+D_t^{(p)}$. After insertion, each term after $p$ receives one additional factor of $\gamma$, giving $\sum_{u=t}^{p}\gamma^{u-t}r_u+\gamma D_t^{(p)}$. Subtracting gives the result.
\end{proof}

Proposition~\ref{prop:length-tax} shows that each additional turn reduces earlier turns' credit by $(1-\gamma)$ times the discounted value of rewards yet to arrive: 5\% per turn at $\gamma=0.95$, and no reduction at $\gamma=1$. The reduction depends only on the number of turns, regardless of what the inserted turn does. This is the length tax described in Section~\ref{sec:discount}.

Discounting also gives the two levels of credit different objectives. The trajectory-level term uses the undiscounted score $S_i=G_{i,1}^{(1)}$, whereas the turn-level term uses discounted returns when $\gamma<1$. Both correspond to the same objective $\mathbb E[G_1^{(1)}]$ only when $\gamma=1$.

\paragraph{Credit differences between equally successful trajectories.}
In an all-success group for the same task, every trajectory has $R=1$ and $\Phi_T=1$. Their scores are equal, so their trajectory-level advantages are zero and the turn-level term supplies the only within-group signal. Figure~\ref{fig:discount-length} plots the return $G_t$ at each turn for two illustrative successful trajectories. One takes 8 turns and passes a check at turns 2, 5, and 8; the other takes 16 turns and passes a check at turns 4, 10, and 16, with each progress event occurring at twice the turn index. At $\gamma=1$, a turn's return depends only on the remaining progress. Both trajectories have the same three-level staircase, with each level lasting twice as long in the longer trajectory. At $\gamma=0.95$, returns also decay with distance to the next reward. The longer trajectory receives less return at each progress stage, with the gap given by Proposition~\ref{prop:length-tax}. Statistics from actual training trajectories agree. We collect 119, 189, and 148 all-success groups from training the $\gamma=0.95$ variant at 4B, 9B, and 35B, respectively. Using the return formula in Section~\ref{sec:credit}, we average turn-level credit within each trajectory and compare trajectories within each group. At $\gamma=0.95$, the shortest trajectory has the highest average credit in 94.1\%, 94.2\%, and 91.9\% of groups, respectively. Recomputing credit on the same trajectories with $\gamma=1$ reduces these proportions to 5.0\%, 6.3\%, and 15.5\%.

\begin{figure}[!htbp]
\centering
\includegraphics[width=0.85\linewidth]{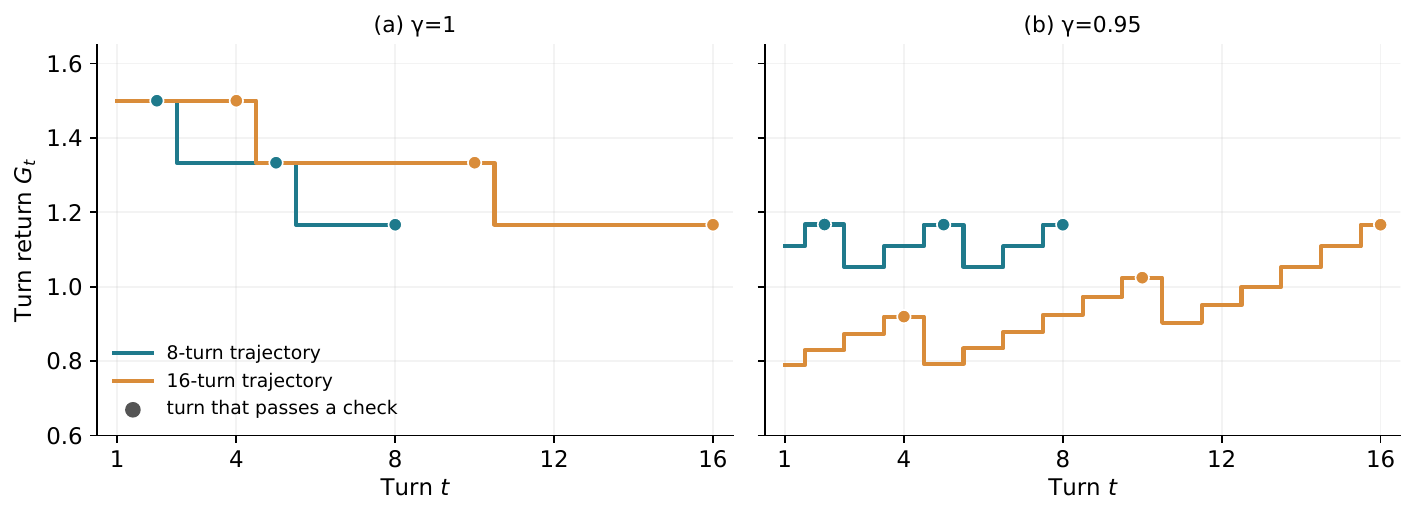}
\caption{Per-turn return $G_t$ for two illustrative successful trajectories with three checks and $c=0.5$.}
\label{fig:discount-length}
\end{figure}

\paragraph{Accuracy losses concentrate on medium and hard tasks.}
We divide test-normal tasks into easy, medium, and hard groups using GRPO's per-task success rates, independently of the methods being compared. Table~\ref{tab:difficulty} compares mean turn counts and success rates for $\gamma=1$, $\gamma=0.95$, and a variant that discounts only all-success groups (V4 in Appendix~\ref{app:group-variants}). Discounting shortens trajectories in all three difficulty groups. Easy tasks already require few turns, leaving little room for shortening and almost no change in success. Medium and hard tasks need more turns; removing these turns lowers success rates. Discounting only all-success groups produces intermediate results at 9B, while at 4B both trajectory length and success rate remain nearly unchanged.

\begin{table}[!htbp]
\centering
\caption{Mean turn counts and success rates (\%) by task difficulty on Test-Normal.}
\label{tab:difficulty}
\small
\setlength{\tabcolsep}{3pt}
\begin{adjustbox}{max width=\linewidth}
\begin{tabular}{llcccccc}
\toprule
& & \multicolumn{2}{c}{\textbf{Easy}} & \multicolumn{2}{c}{\textbf{Medium}} & \multicolumn{2}{c}{\textbf{Hard}} \\
\cmidrule(lr){3-4}\cmidrule(lr){5-6}\cmidrule(lr){7-8}
\textbf{Model} & \textbf{Configuration} & \textbf{Turns} & \textbf{Success} & \textbf{Turns} & \textbf{Success} & \textbf{Turns} & \textbf{Success} \\
\midrule
4B & ProCredit $\gamma=1$ & $18.4$ & $93.3$ & $23.9$ & $73.8$ & $33.7$ & $20.2$ \\
4B & ProCredit $\gamma=0.95$ & $15.5$ & $92.6$ & $20.2$ & $69.0$ & $29.3$ & $17.9$ \\
4B & V4 & $17.9$ & $93.9$ & $22.9$ & $73.9$ & $33.0$ & $19.8$ \\
9B & ProCredit $\gamma=1$ & $16.4$ & $95.2$ & $21.5$ & $78.7$ & $31.2$ & $19.4$ \\
9B & ProCredit $\gamma=0.95$ & $14.0$ & $88.6$ & $18.0$ & $64.8$ & $24.1$ & $9.8$ \\
9B & V4 & $15.7$ & $93.2$ & $20.1$ & $71.1$ & $28.1$ & $15.4$ \\
\bottomrule
\end{tabular}
\end{adjustbox}
\end{table}

\section{Complete comparisons of group-dependent variants}
\label{app:group-variants}
\setcounter{table}{0}
\setcounter{figure}{0}

Section~\ref{sec:discount} uses the same $\gamma$ for all groups. Here we apply discounting, or disable the turn-level term, only for selected types of sampled groups. The four configurations modify only the turn-level term and retain the trajectory-level term.

\begin{table}[!htbp]
\centering
\caption{Group-dependent discounting or removal of turn-level credit on Test-Normal (\%).}
\label{tab:groups}
\normalsize
\renewcommand{\arraystretch}{1.08}
\setlength{\tabcolsep}{8pt}
\begin{tabular}{lll}
\toprule
\multicolumn{3}{l}{\textbf{Turn-level configurations}} \\
\textbf{Variant} & \textbf{All-success groups} & \textbf{Mixed and all-fail groups} \\
\midrule
V1 & $\gamma=0.95$ & $\gamma=0.95$ \\
ProCredit & $\gamma=1$ & $\gamma=1$ \\
V3 & Disabled & $\gamma=0.95$ \\
V4 & $\gamma=0.95$ & $\gamma=1$ \\
\bottomrule
\end{tabular}

\medskip
\setlength{\tabcolsep}{3pt}
\begin{tabular*}{\linewidth}{@{\extracolsep{\fill}}lcccccc@{}}
\toprule
& \multicolumn{2}{c}{\textbf{4B}} & \multicolumn{2}{c}{\textbf{9B}} & \multicolumn{2}{c}{\textbf{35B}} \\
\cmidrule(lr){2-3}\cmidrule(lr){4-5}\cmidrule(lr){6-7}
\textbf{Variant} & \textbf{TGC $\uparrow$} & \textbf{SGC $\uparrow$} & \textbf{TGC $\uparrow$} & \textbf{SGC $\uparrow$} & \textbf{TGC $\uparrow$} & \textbf{SGC $\uparrow$} \\
\midrule
V1 & \score{68.48}{2.53} & \score{45.24}{5.61} & \score{68.20}{5.66} & \score{45.98}{8.65} & \score{65.13}{7.37} & \score{44.64}{8.12} \\
ProCredit & \secondscore{71.53}{0.48} & \secondscore{48.88}{0.45} & \bestscore{78.15}{2.61} & \bestscore{59.82}{3.95} & \bestscore{78.65}{0.85} & \bestscore{59.60}{1.35} \\
V3 & \score{69.84}{1.62} & \score{47.10}{2.35} & \score{68.90}{5.01} & \score{44.57}{7.91} & \score{71.50}{6.37} & \score{51.34}{7.77} \\
V4 & \bestscore{71.70}{3.76} & \bestscore{49.63}{7.03} & \secondscore{73.64}{1.98} & \secondscore{52.68}{3.33} & \secondscore{74.31}{2.99} & \secondscore{57.44}{3.53} \\
\bottomrule
\end{tabular*}
\end{table}

V1 and V4 differ only in whether mixed and all-fail groups are discounted. Removing discounting for these groups improves TGC and SGC at all three model sizes. V4 and ProCredit differ only in whether all-success groups are discounted. ProCredit scores higher at 9B and 35B; they perform similarly at 4B, where V4 varies more across training runs. V3 disables turn-level credit in all-success groups but retains discounting elsewhere, and scores below ProCredit at all three model sizes. The main method retains the turn-level term and uses $\gamma=1$ for every group.

\section{Learning dynamics}
\label{app:learning}
\setcounter{table}{0}
\setcounter{figure}{0}

To examine how progress affects learning, Figure~\ref{fig:learning-steps} compares training and development success rates for GRPO, GRPO-$\Phi$, and ProCredit over the first 50 training steps; all three methods sample the same number of trajectories per step, so steps measure training compute. Training success improves for all three methods and remains similar across them. Differences emerge on the development set, where ProCredit is generally higher during the middle and later stages at 4B and 9B.

\begin{figure}[!htbp]
\centering
\includegraphics[width=0.85\linewidth]{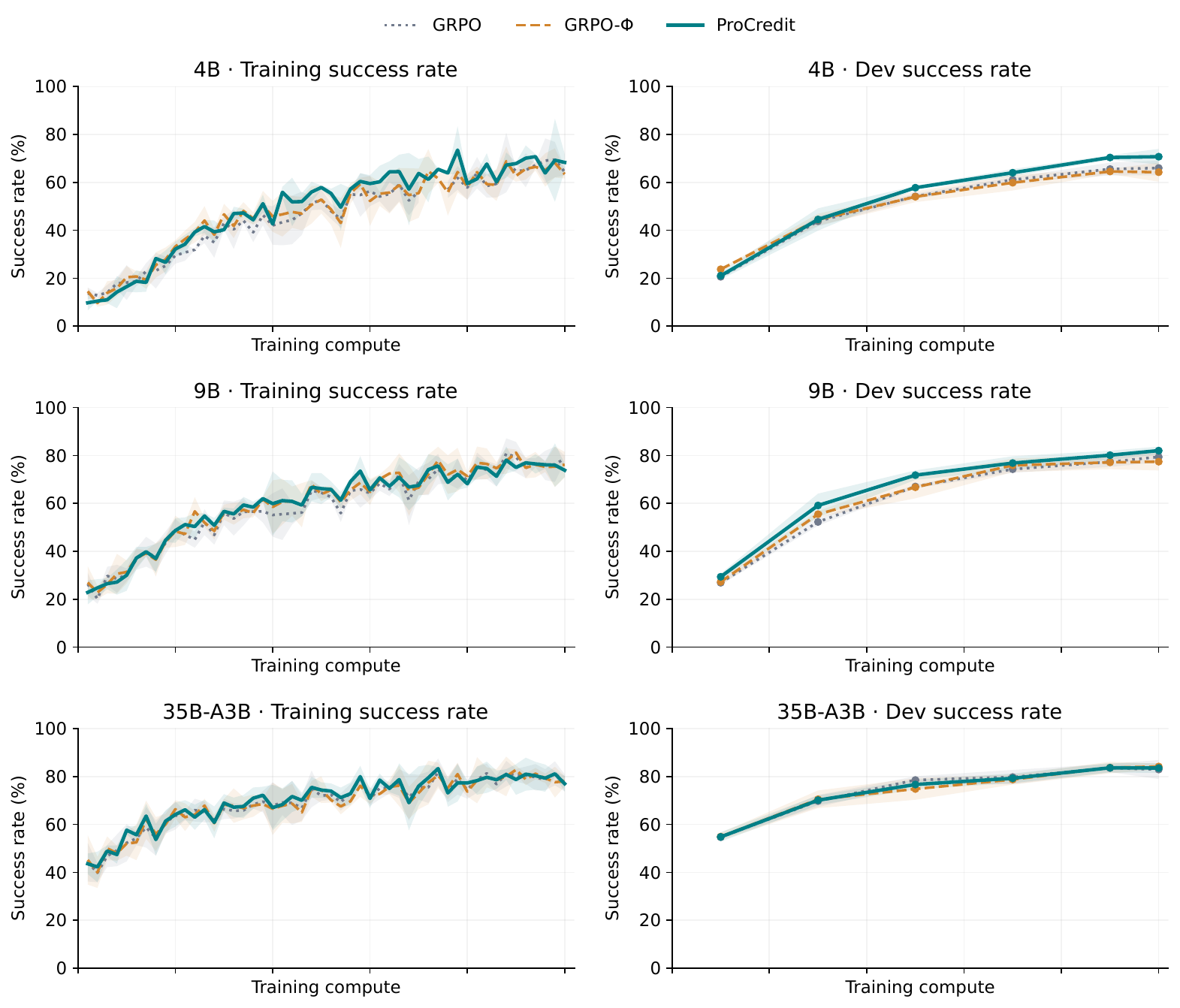}
\caption{Training and development success rates over the first 50 training steps for GRPO, GRPO-$\Phi$, and ProCredit at three model sizes.}
\label{fig:learning-steps}
\end{figure}

Compared with GRPO-$\Phi$, ProCredit's gains appear on held-out tasks. By the later stages of training, development success rises from 64.6\% to 70.4\% at 4B and from 77.1\% to 80.1\% at 9B. GRPO-$\Phi$ matches or exceeds ProCredit's training success at some steps, but these gains do not carry over to the development set.

\paragraph{Comparison by training time.}
Figure~\ref{fig:learning-time} plots the same 4B development checkpoints against cumulative training time. Each run uses four GPUs, and cumulative time excludes initialization, downtime, and development evaluation. Under a six-hour budget, using each run's latest evaluated checkpoint available at that time, ProCredit achieves a mean success rate of 70.39\%, compared with 61.99\% for GRPO and 64.55\% for GRPO-$\Phi$.

\begin{figure}[!htbp]
\centering
\includegraphics[width=0.6\linewidth]{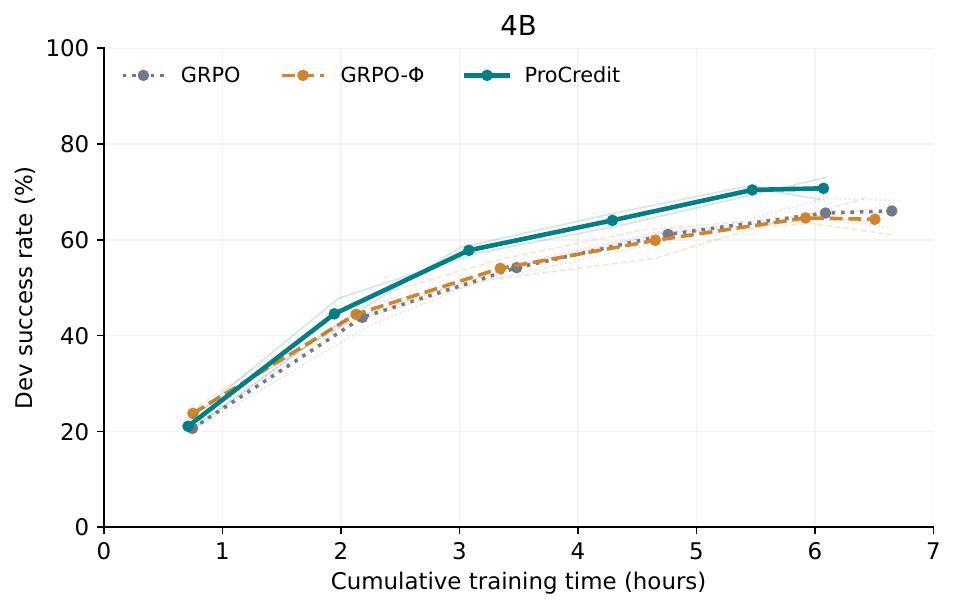}
\caption{4B development success against cumulative training time. Thick curves show seed means at the median time of each checkpoint; faint lines show individual seeds.}
\label{fig:learning-time}
\end{figure}

\section{Retention of BFCL capabilities}
\label{app:bfcl}
\setcounter{table}{0}
\setcounter{figure}{0}

\paragraph{Evaluation scope.}
We use BFCL to assess changes in existing function-calling capabilities after AppWorld training, covering non-live single-turn AST evaluation and four multi-turn subcategories. Categories outside this scope do not enter the aggregate scores. BFCL is not used for training or for selecting checkpoints, seeds, or methods. We evaluate the same checkpoints as in the main AppWorld table.

\begin{table}[!htbp]
\centering
\caption{BFCL scores (\%) before and after AppWorld training.}
\label{tab:bfcl}
\small
\setlength{\tabcolsep}{3pt}
\begin{adjustbox}{max width=\linewidth}
\begin{tabular}{lcccc}
\toprule
& \multicolumn{2}{c}{\textbf{4B}} & \multicolumn{2}{c}{\textbf{9B}} \\
\cmidrule(lr){2-3}\cmidrule(lr){4-5}
\textbf{Method} & \textbf{Single-turn AST $\uparrow$} & \textbf{Multi-turn $\uparrow$} & \textbf{Single-turn AST $\uparrow$} & \textbf{Multi-turn $\uparrow$} \\
\midrule
Base model & \scoreonly{85.73} & \scoreonly{52.38} & \scoreonly{87.19} & \scoreonly{52.25} \\
GRPO & \score{86.52}{0.30} & \score{53.54}{1.23} & \score{86.30}{0.22} & \score{52.21}{1.60} \\
GRPO-$\Phi$ & \score{85.28}{0.06} & \score{52.50}{1.25} & \score{86.47}{0.30} & \score{53.21}{0.69} \\
ProCredit & \score{85.94}{0.10} & \score{53.04}{1.01} & \score{86.73}{0.17} & \score{53.83}{1.94} \\
\bottomrule
\end{tabular}
\end{adjustbox}
\end{table}

\paragraph{General function-calling capabilities remain stable.}
All three trained methods score close to the base models: single-turn AST scores differ by less than 1 percentage point and multi-turn scores by less than 2 points, comparable to the variation across three training runs. AppWorld training does not weaken the function-calling capabilities measured by BFCL. The ordering of methods varies by model size, and we do not rank them here.

\section{Two properties of the trajectory score}
\label{app:score-properties}
\setcounter{table}{0}
\setcounter{figure}{0}

We use the notation from Sections~\ref{sec:progress} and~\ref{sec:credit}.

\begin{proposition}
\label{prop:success-order}
For any $c>0$, every successful trajectory for a task has a strictly higher score than every failed trajectory for that task.
\end{proposition}

\begin{proposition}
\label{prop:optimal-policy}
If a policy can complete the task with probability 1, the set of policies maximizing expected trajectory score is identical to the set maximizing success probability.
\end{proposition}

The proofs use three facts: all trajectories for a task share the same $\Phi_0$; $\Phi_T\in[0,1]$; and $R=1$ implies $\Phi_T=1$, because the outcome reward requires all acceptance checks to pass and $\mathcal C_g$ is a subset of those checks.

\begin{proof}[Proof of Proposition~\ref{prop:success-order}]
A successful trajectory has score $S=1+c\,(1-\Phi_0)$. A failed trajectory has $R=0$ and $\Phi_T\le1$, so its score satisfies $S\le c\,(1-\Phi_0)$. The difference is at least 1, independently of $c$.
\end{proof}

\begin{proof}[Proof of Proposition~\ref{prop:optimal-policy}]
For any policy $\pi$,
\begin{equation}
\mathbb E_\pi[S]
=\mathbb E_\pi[R]+c\,(\mathbb E_\pi[\Phi_T]-\Phi_0)
\le1+c\,(1-\Phi_0)=:S_{\max}.
\end{equation}
Equality holds if and only if $\mathbb E_\pi[R]=1$. Necessity follows because $\mathbb E_\pi[R]$ and $\mathbb E_\pi[\Phi_T]$ are each bounded above by 1, and equality requires both to reach that bound. Sufficiency follows because $R=1$ implies $\Phi_T=1$. When the task is solvable, $S_{\max}$ is attainable. Thus, the policies maximizing $\mathbb E[S]$ are exactly those that succeed with probability 1, which are also exactly the policies maximizing $\mathbb E[R]$.
\end{proof}

Two points qualify these results. First, Proposition~\ref{prop:optimal-policy} does not hold when the task is unsolvable: expected score is success probability plus $c$ times expected completion, so optimizing it trades success probability for partial completion at a rate set by $c$. This trade-off is the purpose of scoring partial completion. Second, the policy-invariance theorem of \citet{ng1999} requires zero potential at terminal states, which corresponds to subtracting $c\,\Phi_T$ at the final turn. The trajectory score then becomes $R-c\,\Phi_0$: optimal policies remain unchanged, but trajectory-level progress information disappears, leaving all-fail groups with zero advantages again. We retain $\Phi_T$ to preserve this information. AppWorld has deterministic environments and a reference solution for every task, so the premise of Proposition~\ref{prop:optimal-policy} holds for the training tasks.

\section{Why turns with different histories can share one reference value}
\label{app:pooling}
\setcounter{table}{0}
\setcounter{figure}{0}

GRPO compares trajectories sampled for the same task, all starting from the same initial state. ProCredit's turn-level term pools all turns for the task and subtracts their mean, even though different trajectories have different histories by turn $t$. This section explains why this comparison remains valid and how it relates to state-based grouping.

Both terms subtract a reference computed from the sampled group. For the trajectory-level term, including a trajectory's own score in the group mean scales the expected policy gradient by $(N-1)/N$, as in GRPO. The turn-level term subtracts the mean return over all turns of all trajectories for the same task. A turn's reference need not come from turns with the same history: the turns of the other trajectories are sampled independently of the action taken at this turn, so their mean $b_{-i}$ is a valid baseline for every turn of trajectory $i$. Equation~\eqref{eq:turn-advantage} uses the pooled mean, which differs from $b_{-i}$ by $w_i(\bar G_i-b_{-i})$, where $w_i=T_i/|\mathcal U_g|$ is the share of turns contributed by trajectory $i$, about $1/N$, and $\bar G_i$ is its mean turn return. This term varies with the trajectory's outcome, progress, and length, and is a bounded correction of order $1/N$. Equation~\eqref{eq:adjacent-credit} is unaffected, since it requires only that every turn for a task subtract the same value. The resulting turn-level advantage measures how much the return from that turn exceeds the mean over all turns for the task.

State-based grouping, as in GiGPO with $\gamma=1$ in Table~\ref{tab:granularity}, changes the reference value. It groups turns at the same progress state for the same task and subtracts the group's mean, with the aim of evaluating arrival at a state separately from the current action. The two turn-level advantages differ by one term:
\begin{equation}
A^{\mathrm{anchor}}_{i,t}=A^{\mathrm{turn}}_{i,t}
-\bigl(\mu_{\mathcal A(i,t)}-\mu_{\mathcal U_g}\bigr),
\end{equation}
where $\mu_{\mathcal A}$ is the mean for the turn's anchor group and $\mu_{\mathcal U}$ is the mean over all turns for the task. The relative ordering of turns within the same state is identical. On AppWorld, anchor groups have a median size of 6.6 turns and fewer than 1\% are singletons, so group size is not the issue. Yet this method does not outperform ProCredit at any of the three model sizes. There are two reasons. Once turns are grouped by state, credit for reaching a good state is no longer assigned to the current turn; earlier turns must recover that credit through comparisons within their own groups, which requires different continuations within those groups. Also, the anchors record only passed checks, omitting database contents and recent actions. Turns in the same group may therefore face different decisions, and a reference value estimated from just a few turns varies more than the mean over all turns for the task.

\section{Training settings}
\label{app:training-settings}

Table~\ref{tab:training-settings} lists the training settings of ProCredit behind Tables~\ref{tab:main} and~\ref{tab:toolsandbox}. Length limits and the number of GPUs grow with model size; the ToolSandbox settings differ from AppWorld only in tasks per step, training steps, and the turn limit. Baselines use the hyperparameters of their original papers.

\begin{table}[!htbp]
\centering
\caption{Training settings of ProCredit.}
\label{tab:training-settings}
\small
\setlength{\tabcolsep}{8pt}
\begin{tabular}{lrr}
\toprule
\textbf{Setting} & \textbf{AppWorld} & \textbf{ToolSandbox} \\
\midrule
Rollouts per task & 8 & 8 \\
Tasks per step & 25 & 16 \\
Training steps & 50 & 35 \\
Learning rate & 1e-6 & 1e-6 \\
PPO clip range & 0.2 & 0.2 \\
KL term / entropy term & none & none \\
Loss aggregation & token-mean & token-mean \\
Maximum turns per trajectory & 50 & 30 \\
Maximum tokens per turn & 2048 & 2048 \\
Maximum prompt length & 4096 & 8192 \\
Maximum response length & 28672--32768 & 28672 \\
Context length & 32768--36864 & 32768 \\
Sampling temperature & 1.0 & 1.0 \\
Progress coefficient $c$ / discount $\gamma$ & 0.5 / 1 & 0.5 / 1 \\
GPUs per run & 4--8 & 2 \\
\bottomrule
\end{tabular}
\end{table}

\end{document}